\documentclass[11pt]{article}

\usepackage[colorlinks=true, urlcolor=blue, pdfborder={0 0 0}]{hyperref}

\newcommand{\epsgen}{\epsilon_{\textrm{g}}}
\newcommand{\epspr}{\epsilon_{\textrm{p}}}
\newcommand{\deltapr}{\delta_{\textrm{p}}}
\newcommand{\deltagen}{\delta_{\textrm{g}}}

\newcommand{\tO}{\tilde{O}}

\newcommand{\SU}{\mathcal{SU}}
\newcommand{\LS}{\mathcal{LS}}
\newcommand{\ULS}{\overbar{\mathcal{LS}}}

\newcommand{\GAP}{\mathcal{GAP}}

\title{Smooth Sensitivity Based Approach for Differentially Private Principal Component Analysis}
\author{Ran
  Gilad-Bachrach\footnote{rang@microsoft.com, Microsoft Research} \and Alon Gonen\footnote{agonen@cs.princeton.edu, Computer Science,
    Princeton University. Part of this work was conducted when AG was with Microsoft Research}}
\date{\today}
\usepackage{graphicx,subfigure}
\usepackage{epsfig}
\usepackage{hyperref}
\usepackage{amsmath}
\usepackage{amssymb}
\usepackage{algorithm}
\usepackage{algorithmic}
\usepackage{array}
\usepackage{eqparbox}
\usepackage{url}
\usepackage{enumerate}
\usepackage{amsfonts}
\usepackage{boxedminipage}
\usepackage{xcolor}
 \usepackage{framed}  
\usepackage{rotating}
\usepackage{array}
\usepackage{multirow}
\usepackage{color}
\usepackage{tikz}
\usepackage{mathabx}
\usepackage{tabularx,ragged2e,booktabs,caption}

{
\begin{center}
\begin{boxedminipage}{0.8\linewidth}
\begin{center}
\textbf{\texttt{#1}}
\end{center}
\rm
\begin{tabbing}
....\=...\=...\=...\=...\=  \+ \kill
} %
{\end{tabbing} 
\end{boxedminipage} \end{center} 
}

{
\vspace{0.01cm}
\begin{center}
\begin{boxedminipage}{0.8\linewidth}
\begin{center}
\textbf{\texttt{#1}}
\end{center}
} %
{
\end{boxedminipage} \end{center} \vspace{0.3cm}
}

\newtheorem{definition}{Definition}
\newtheorem{lemma}{Lemma}
\newtheorem{corollary}{Corollary}
\newtheorem{theorem}{Theorem}

\newtheorem{remark}{Remark}

\newcommand{\bE}{\mathbb{E}}
\newcommand{\reals}{\mathbb{R}}

\newcommand{\bN}{\mathbb{N}}

\newcommand{\supp}{\mathrm{supp}}

\newcommand{\tu}{\tilde{u}}

\newcommand{\tU}{\tilde{U}}

\newcommand{\hu}{\hat{u}}

\newcommand{\hU}{\hat{U}}

\newcommand{\hF}{\hat{F}}
\newcommand{\hC}{\hat{C}}

\newcommand{\cX}{\mathcal{X}}
\newcommand{\cD}{\mathcal{D}}

\newcommand{\cU}{\mathcal{U}}

\newcommand{\cA}{\mathcal{A}}

\newenvironment{proof}{\par\noindent{\bf Proof\ }}{\hfill\BlackBox\\[2mm]}
\newcommand{\BlackBox}{\rule{1.5ex}{1.5ex}}  
\newcommand{\overbar}[1]{\mkern 1.5mu\overline{\mkern-1.5mu#1\mkern-1.5mu}\mkern 1.5mu}

\makeatletter
\def\moverlay{\mathpalette\mov@rlay}
\def\mov@rlay#1#2{\leavevmode\vtop{%
   \baselineskip\z@skip \lineskiplimit-\maxdimen
   \ialign{\hfil$\m@th#1##$\hfil\cr#2\crcr}}}
\newcommand{\charfusion}[3][\mathord]{
    #1{\ifx#1\mathop\vphantom{#2}\fi
        \mathpalette\mov@rlay{#2\cr#3}
      }
    \ifx#1\mathop\expandafter\displaylimits\fi}
\makeatother

\newcommand{\hL}{\hat{L}}

\newcommand{\tr}{\mathrm{tr}}

\DeclareMathOperator*{\argmin}{arg\,min}

\renewcommand{\eqref}[1]{Equation~(\ref{#1})}

\newcommand{\thmref}[1]{Theorem~\ref{#1}}
\newcommand{\lemref}[1]{Lemma~\ref{#1}}

\newcommand{\corref}[1]{Corollary~\ref{#1}}

\newcommand{\algref}[1]{Algorithm~\ref{#1}}

\newcommand{\handout}[5]{
   \renewcommand{\thepage}{#1-\arabic{page}}
   \noindent
   \begin{center}
   \framebox{
      \vbox{
    \hbox to 5.78in { {\bf (67577) Introduction to Machine Learning}
         \hfill #2 }
       \vspace{4mm}
       \hbox to 5.78in { {\Large \hfill #5  \hfill} }
       \vspace{2mm}
       \hbox to 5.78in { {\it #3 \hfill #4} }
      }
   }
   \end{center}
   \vspace*{4mm}
}

\begin{document}
\maketitle

\begin{abstract}
We consider the challenge of differentially private PCA. 
Currently known methods for this task either employ the computationally intensive
\emph{exponential mechanism} or require an access to the covariance matrix,
and therefore fail to utilize potential sparsity of the data. The problem of
designing simpler and more efficient methods for this task has been
raised as an open problem in \cite{kapralov2013differentially}.

In this paper we address this problem by employing the output
perturbation mechanism. Despite being arguably the simplest and most
straightforward technique, it has been overlooked due to
the large \emph{global sensitivity} associated with publishing the
leading eigenvector. We tackle this issue by adopting a \emph{smooth sensitivity} based 
approach, which allows us to
establish differential privacy (in a worst-case manner) and
near-optimal sample complexity results under eigengap assumption. We
consider both the pure and the approximate  notions of differential privacy, and demonstrate a tradeoff between privacy level and sample complexity. We conclude by
suggesting how our results can be extended to related problems.
\end{abstract} 

\section{Introduction} \label{sec:intro}
\emph{Differential Privacy} has become a crucial requirement in many
machine learning tasks involving private data such as medical and
financial records (\cite{dwork2008differential,
  dwork2014algorithmic, chaudhuri2011differentially, blum2013learning,
  mcsherry2007mechanism}). Informally speaking, a mechanism is said to be
differentially private if one can hardly distinguish between two
outputs of the algorithm corresponding to samples that differ in one
entry. Since each entry typically corresponds
to records of a single person, differential privacy essentially
requires that the participation of a single individual in the sample
(e.g. medical tests) would not reveal its private
information. This requirement inherently implies a tradeoff between
privacy and accuracy. Accordingly, considerable efforts have
been made to identify structural properties that enable us to
reduce this conflict.

\emph{Principal component analysis} (PCA) is a fundamental
dimensionality reduction technique in machine
learning and data science. Finding a low-rank approximation of a given
dataset is beneficial in terms of time and space complexity. In some
scenarios (e.g. vision tasks), it also has the benefit of noise
removal. 

In view of the above, it is not surprising that differentially private
PCA has received substantial attention recently
(\cite{blum2005practical, chaudhuri2012near,
  kapralov2013differentially,hardt2013beyond, dwork2014analyze}). 
  
Our main contribution is a simple yet efficient method to make PCA differentially private. In a nut-shell,
our method modifies standard PCA algorithm by adding a post-processing
step in which a suitable noise is added to the output. Therefore, it is straightforward to combine it with any PCA implementation, including 
implementation that make use of unique properties of the data such as
sparsity. To achieve that, we show that if there is a large eigengap
between the leading eigenvalues of the covariance matrix, then the PCA problem 
becomes less sensitive to changes in its inputs. Hence, we can compute the amount of noise to inject as a function 
of the eigen gap.

 \subsection{Problem Definition}
  Let us now describe the considered problem formally.
Let $\cD$ an unknown distribution
defined on the unit ball in $\reals^d$.\footnote{Our results can be
easily scaled to balls of larger radius.} Given a low-rank parameter $k
\in [d]$, our ultimate goal is to 
approximately solve
\begin{equation} \label{eq:stochasticPCA}
\min_{U \in \cU} F(U) = -\tr(UU^\top \, C),~\textrm{where}~C:=\bE_{x \sim \cD}[xx^\top],~~\cU = \{U \in \reals^{d \times k}:~U^\top U = I_k\}~,
\end{equation}
while preserving differential privacy. The input to the learning algorithm $\cA$
consists of a sample $S=(x_1,\ldots,x_n)$ drawn i.i.d. according to
$\cD$. Its output is denoted by $\hU \in \cU$. The sample complexity of the algorithm
is a function $n: (0,1)^3 \rightarrow \bN$, where
$n(\epsgen,\epspr,\deltapr)$ is the minimal size of an i.i.d. sample
$S=(x_1,\ldots,x_n) \sim \cD^n$ for which
the following conditions simultaneously hold:\\

\noindent $\boldsymbol{\epsgen}$-\textbf{accuracy:
}\footnote{Given a confidence parameter $\delta$, standard techniques
  can be used to decrease the probability of failure to $\delta$ while
  incurring only logarithmic overhead in terms of sample
  complexity.}with constant probability over both the draw of the sample $S$
according to $\cD^n$ and the internal randomness of the
algorithm,
$$
F(\hU) \le \min_{U \in \cU} F(U) + \epsgen~.
$$
\noindent $\boldsymbol{(\epspr,\deltapr)}$-\textbf{differential privacy: } Let
$d(S,S')$ be the minimal number of elements 
that should be removed or added to the sample $S'$
to obtain the sample $S$. We say that $S$ and $S'$ are
neighboring samples if $d(S,S') \le 1$. We require that for all neighboring
samples $S,S'$, and for all  $U \in \cU$,
\begin{equation} \label{eq:dp}
p(\cA(S) = U) \le \exp(\epspr) p(\cA(S') = U) + \deltapr~,
\end{equation}
(where $p$ refers to the density function). The stricter notion of ``pure'' differential privacy requires also
that $\deltapr=0$.

\section{Algorithms and Main Result}
In this paper we focus on particularly simple and efficiency-preserving method, named output
perturbation. As its name suggests, the basic idea is to add
noise to the output of an (approximately) exact algorithm. Arguably, this is the 
simplest and most flexible method, as it can be applied to any algorithm
in a black-box fashion while preserving its efficiency.  We assume an access to an algorithm $\cA$ which (approximately)
minimize the \emph{empirical risk} 
$$
\hF(U) := -\tr(UU^\top \hC),~~\hC:=\frac{1}{n} \sum_{i=1}^n x_i x_i^\top~.
$$ 
We also assume that
the algorithm $\cA$ outputs the gap between the $k$-th and the
$(k+1)$-th eigenvalues of $\hat{C}$. This assumption is not
restrictive, as every reasonable PCA solver possesses this
capability. Based on the output of $\cA$,
our mechanism determines the noise level.  The main challenge in our work is to set the noise
level so that differential privacy holds for any sample, and high
accuracy is achieved under eigengap assumption. 

Before adding the noise, there is another subtle issue which should be carefully addressed. 
To illustrate this challenge, consider the case $k=1$. Clearly, a unit
vector $u \in \reals^d$ is a leading eigenvector if and only if $-u$ is also a leading eigenvector. Since the sign 
of the vector is arbitrary, a PCA solver might use it to leak private information, such as whether a specific point $x^*$
was in the dataset ot not. Overcoming this potential risk is possible by negating the output of the PCA solver with
probability $1/2$ before adding the noise. More generally, for the case $k>1$, we will replace
$\hU$ by $R \hU$, where $R \in \reals^{d \times d}$ is a random
orthogonal matrix. We then add the noise and perform QR decomposition
to obtain the final output.
 A detailed pseudocode of our method is given in \algref{alg:dpPCA}.
\begin{algorithm}
\caption{Differentially private PCA using Output perturbation}
\label{alg:dpPCA}
\begin{algorithmic}
\STATE
\textbf{Parameters: } $\epsgen, \deltagen, \epspr, \deltapr \in
(0,1),~k \in [d], \texttt{PURE} \in \{\texttt{TRUE}, \texttt{FALSE}\}$
\STATE
\textbf{Input: } $\hU:=\argmin_{U \in \cU} -\tr(U \hat{C}),~~G =
\lambda_1(\hat{C}) - \lambda_2(\hat{C})$
\STATE
\textbf{Oracle: } $\cA(S) = (\tU:=\argmin_{U \in \cU} \hF(U),
~\lambda_k(\hat{C}) - \lambda_{k+1}(\hat{C}))$
\STATE
Draw a random orthogonal matrix $R \in \reals^{d \times d}$ 
\STATE
Replace $\hU$ with $\overbar{U} = RU$ 
\IF {$\texttt{PURE}=\texttt{TRUE}$} 
\STATE Draw $E:=E_{\texttt{PURE}}  \in \reals^{d \times k}$ as
described in \eqref{eq:finalNoisePure}
\ELSE
\STATE Draw $E:=E_{\texttt{APPROX}}  \in \reals^{d \times k}$ as
described in \eqref{eq:finalNoiseNonPure} 
\ENDIF
\STATE Return the matrix $\tU = QR(\overbar{U} + E)$.
\end{algorithmic}
\end{algorithm}
To simplify the presentation and for the sake of conciseness, we focus
on the case $k=1$. The case of $k>1$ is a straightforward extension of the case $k=1$
\begin{theorem} \label{thm:mainApprox}  \textbf{(Main theorem: approximate case)}
Given that $\texttt{PURE} = \texttt{FALSE}$, \algref{alg:dpPCA} is $(\epspr, \deltapr)$-differentially private. Furthermore, if
$\GAP(\cD):=\lambda_1(\bE[xx^\top])-\lambda_2(\bE[xx^\top])>0$, 
then its sample complexity is at most\footnote{We use the $\tilde{O}$
  notation to hide logarithmic dependencies.}
\[
n(\epsgen, \epspr, \deltapr) = \tO \left(\frac{\sqrt{d}}{\GAP(\cD)  \epspr \epsgen}  \right)
\]
\end{theorem}
\begin{theorem} \label{thm:mainPure}  \textbf{(Main theorem: pure case)}
Given that $\texttt{PURE} = \texttt{TRUE}$, \algref{alg:dpPCA} is $\epspr$-differentially private. Furthermore, if
$\GAP(\cD):=\lambda_1(\bE[xx^\top])-\lambda_2(\bE[xx^\top])>0$, then
its sample complexity is at most
\[
n(\epsgen, \epspr) = \tO \left(\frac{d^{3/2}} {\GAP(\cD)  \epspr \epsgen}  \right)
\]
\end{theorem}

\section{Related Work}
Differentially private PCA has been extensively investigated in
\cite{chaudhuri2012near, hardt2013beyond, kapralov2013differentially,
  blum2005practical, dwork2014analyze, hardt2012beating}. The lower
bound of \cite{dwork2014analyze} implies that our sample complexity for
the approximate case (see \thmref{thm:mainApprox}) is optimal up to logarithmic factors. For the
pure case, the lower bound given by \cite{chaudhuri2012near} scales
with $d$, whereas our upper bound (see \thmref{thm:mainPure}) scales
with $d^{3/2}$. 

The first proposed method for differential
private PCA was Sub-Linear Queries (SULQ) (\cite{blum2005practical}). It employs the general strategy of
\emph{input perturbation} by adding random Gaussian noise to the
empirical covariance matrix. Both the algorithm and its analysis have
been refined recently by \cite{dwork2014analyze}. Restating their
results within our framework gives approximate differential privacy
with sample complexity bound identical to \thmref{thm:mainApprox}. They also
consider the gap-free scenario. As we mentioned previously, the main
limitation of this method is that it requires an access to the
covariance matrix, which might be too costly in terms of space and time. Many fast PCA implementations
(e.g. \cite{shamir2015convergence, Ghashami2016, clarkson2013low,
  jain2016streaming, jin2015robust}) avoid working
with the covariance matrix and consequently utilize potential sparsity
of the data. As we mentioned previously, our output perturbation can be combined with any
of these methods. 

Another approach that has been investigated in
\cite{chaudhuri2012near, kapralov2013differentially} is to use the
exponential mechanism (\cite{dwork2014algorithmic}). While this
approach achieves pure differential with optimal sample
complexity (also in the gap-free case), the only theoretically analyzed implementation of the associated
sampling method runs in time $O(d^6)$. 

Besides the spectral gap assumption, another common approach is to assume
some form of incoherence. This route has been taken by
\cite{hardt2013beyond, hardt2012beating} who provide several
interesting differentially private methods for PCA.
\section{Analysis}
In this section we prove our main result. We start by defining the
local and global sensitivity of PCA, and proceed to define and
analyze the smooth sensitivity.
\subsection{Local and Global Sensitivity up To Equivalence}
In the context of output perturbation, the sensitivity of a sample is
defined as the maximum distance between two outputs of
PCA corresponding to the neighboring samples. Unless specified
otherwise, the distance is measured according to the
$\ell_2$-norm. Due to the equivalence between outputs discussed above,
it makes sense to define the notion of distance between equivalent solutions. Namely, for any $U \in \cU$,
we define $[U] = \{RU:~R \in \reals^d,~R^\top R = I_d\}$. The distance
between $[U]$ and $[V]$ is defined by $\|[U]-[V]\| = \min\{
\|U'-V'\|:~U' \in [U],~V' \in [V]\}$. Since our algorithm replaces the
output $\cU$ of PCA by $R\hU$, where $R$ is a random orthogonal
matrix, this modification does not harm our analysis.
\begin{definition} \textbf{(Global and local sensitivity)} \label{def:ell2sense}
The $\ell_2$-local sensitivity of a PCA algorithm $\cA: \cX^n
\rightarrow \cU$ w.r.t. a sample $S=(x_1,\ldots,x_n)$ is defined as
\[
\LS(S) := \LS_{\cA}(S) = \max_{S': d(S,S') \le 1}  \|[\cA(S)]-[\cA(S')]\|~.
\]
The global sensitivity of $\cA$ is defined as $\sup \{LS(S):~S \in
\supp(\cD^n) \}$. The $\ell_1$-local sensitivity is defined analogously.
\end{definition}
It is known that adding noise proportional to the global sensitivity
(using a suitable noise distribution depending on the privacy parameters)
yields differential privacy (\cite{dwork2014analyze}). The following example due to \cite{chaudhuri2012near} illustrates the difficulty in preserving both accuracy and
privacy using this approach. Let $u,u' \in \reals^d$ be two
orthonormal vectors and consider two
samples $S$ and $S'$, where $S$ consists of $n+1$ copies of $u$ and
$n$ copies of $u'$, whereas $S'$ consists of $n+1$ copies of $u'$ and
$n$ copies of $u$. The leading eigenvectors associated with $S$ and
$S'$ are $u$ and $u'$, respectively. To satisfy differential privacy
in this case, one should inject a noise proportional to the distance
between $u$ and $u'$. In particular, the amount of noise 
does not decreases as a function of the sample size, hence accuracy
can not be preserved. 

An easy computation shows that the eigengap in the previous examples
scales inversely with the sample size. The following theorem shows
that the larger the eigengap the smaller is the local sensitivity. We
first make the following definition.
\begin{definition}
Given a sample $S =
(x_1,\ldots,x_n)$, we denote the eigengap between the two leading
eigenvalues of the empirical covariance matrix $\frac{1}{n}
\sum_{i=1}^n x_i x_i^\top$ by $\GAP(S)$.
\end{definition}
\begin{theorem} \label{thm:ell2SensitivityUpper}
Let $S=(x_1,\ldots, x_n) \in \supp(\cD^n)$ be a sample and suppose
$\GAP(S)>0$. Then there exists a global constant $C>0$
such that the $\ell_2$-sensitivity of PCA is at most
$\frac{3C}{n \cdot \GAP(S)}$. Furthermore, the global
$\ell_2$-sensitivity is $\sqrt{2}$. The $\ell_1$-local sensitivity is
at most $\sqrt{d}$ times larger than the $\ell_2$-local sensitivity,
and the $\ell_1$-global sensitivity is at most $2$.
\end{theorem}
This result can be proved in several ways. The approach taken here
exploits recent results on strict saddle problems, which include PCA
as a special case. 
\begin{proof}
Let $S =(x_1,\ldots, x_n) \in \supp(\cD^n)$ and $S' = (x_1,
\ldots, x_{n-1}) \in \supp(\cD^{n-1})$ be two neighboring samples and
let $u,u'$ be the minimizers of the corresponding empirical
risks. Denote by $\hC=n^{-1} \sum_{i=1}^nx_i x_i^\top$ and
$\hC'=n^{-1} \sum_{i=1}^nx_i x_i^\top$. By
KKT conditions (\cite{borwein2010convex}), there exist
$\lambda:=\lambda(u)$ and $\lambda'=\lambda(u')$ such that, 
\[
u=\argmin_{v \in \reals^d}  \underbrace{-v ^\top C v+\lambda
  (\|v\|^2-1)}_{=:\hL(v)},~~~u'=\argmin_{v \in \reals^d}
\underbrace{-v^\top \hC' v+\lambda' (\|v\|^2-1)}_{=:\hL'(v)}
\]
Also, $\lambda$ and $\lambda'$ admit the closed forms: 
\[
\lambda = u^\top \hC u,~~\lambda' =
u'^\top \hC' u'~.
\]
That is, $\lambda$ is the leading eigenvalue of $\hC$ and $\lambda'$
is the leading eigenvalue of $\hC'$. By first-order conditions, both $\nabla \hL(u)=0$ and $\nabla
\hL'(u')=0$. Also,   
\[
\nabla \hL(u') = -\hC u' + \lambda u'
= \frac{n-1}{n}  \nabla \hL'(u') - n^{-1}x_n x_n^\top u - \frac{n-1}{n}
\lambda' u' + \lambda u' =n^{-1} x_n x_n^\top u -(\lambda'-\lambda) u + n^{-1}\lambda' u'
\]
Since $\|x_i\|\le 1$ for all $i$, by using Weyl's inequality we obtain that
$\|\nabla \hL(u')\|\le \frac{3}{n}$. 

We next use the strict saddle property of PCA to bound the distance
between $u$ and $u'$. Concretely, it is shown in \cite{gonen2017fast} that our formulation of PCA is
$(\alpha,\gamma,\tau)$-strict saddle with $\alpha, \gamma, \tau =
C^{-1}\GAP(S)$ for some constant $C>0$ (see \cite{gonen2017fast}). 
Theorem 5 in this paper implies that if $n = \Omega(1/G^2)$, then $u'$
lies in a $C^{-1}\GAP(S)$-strongly convex region (of the objective $\hat{F}$) around the
minimizer $u$. By strong convexity (see \cite{nesterov2004introductory}), 
\[
\nabla \hL(u') ^\top(u-u') = (\nabla \hL(u') - \nabla \hL(u))^\top (u-u') \ge C^{-1}\GAP(S) \|u-u'\|^2~.
\]
Using Cauchy-Schwarz inequality, we obtain that
\[
\|\nabla \hL(u')\| \, \|u-u'\| \ge C^{-1}\GAP(S) \|u-u'\|^2 \Rightarrow \|u-u'\| \le
\frac{C}{\GAP(S)} \|\nabla \hL(u')\| \le  \frac{3C}{n\GAP(S)}~.
\]
The bound on the $\ell_2$-local sensitivity follows immediately. The bound on the $\ell_1$-local sensitivity follows
from the fact that the $\ell_1$-distance is at most $\sqrt{d}$ larger
than the $\ell_2$-distance. The bounds on the global sensitivity are
simply the $\ell_2$ and the $\ell_1$ distances between two
perpendicular unit vectors.
\end{proof}
It is tempting to replace the global sensitivity with the local one
in hope of ensuring differential privacy in a worst case manner
and achieving high accuracy under the common eigengap assumption. In general, this approach is
problematic since the local sensitivity itself might be sensitive.\footnote{The
following example due to \cite{dwork2014algorithmic} illustrates this
idea. Suppose that we would like to compute the median of a
given sequence in a differential private manner. Let $S$ be a sample
consisting of $n/2+1$ zeros and $n/2$ elements of magnitude at least
$10^6$. Assuming that we break ties in favor of the smaller value, the
local sensitivity of $S$ is zero. On the other hand, by removing a
single zero element from $S$, we obtain a neighboring sample whose
local sensitivity is at least $10^6$.} This brings us to the notion of
\emph{smooth sensitivity}, which we describe in the next part.

\subsection{Background on Smooth sensitivity}
Originated in (\cite{Nissim2007}), the notion of smooth sensitivity provides a systematic
framework for generating insensitive surrogate for the local
sensitivity. It consists of two main ingredients: a) Finding a
suitable smooth upper bound on the local sensitivity. b) Generating
noise according to an \emph{admissible} distribution scaled by
the smooth upper bound.
\begin{definition} \textbf{(smooth upper bound on the local
    sensitivity (\cite{Nissim2007}))} \label{def:sub}
For $\beta>0$, a function $\SU: \bigcup_{n \in \bN} \supp(\cD^n) \rightarrow
\reals_{\ge 0}$ is a $\beta$-smooth upper bound on the local
sensitivity $\LS$ if it satisfies the following conditions:
\begin{enumerate}
\item
$\SU(S) \ge \LS(S)$ for every sample $S$ 
\item
For every two neighboring samples $S,S'$, 
$$
\exp(-\beta) \SU(S') \le \SU(S) \le \exp(\beta) \SU(S')~.
$$
\end{enumerate}
\end{definition}
The following characterization of the smooth sensitivity is often useful.
\begin{definition} \label{def:ak}
Let $\ULS$ be any point-wise upper bound on the local sensitivity. For
a sample $S$ and $k \in \bN$, we define
\[
A^{(k)}(S) = \max_{S':~d(S,S') \le k} ~\ULS(S')~.
\]
\end{definition}
\begin{lemma} \textbf{(\cite{Nissim2007}[Claim 3.2])} \label{lem:nissim}
Let $\ULS$ be any point-wise upper bound on the local sensitivity and
define $A^{(k)}(S)$ as above. The function $U: \bigcup_{n \in \bN} \supp(\cD^n)
\rightarrow \reals$ defined by
\[
U(S)= \max_{k \in [n]} \exp(-\beta k) A^{(k)}(S) 
\]
is a $\beta$-smooth upper bound on the local sensitivity.
\end{lemma}
Analogously to the global sensitivity, the smooth local
sensitivity determines the scale of the noise associated with our
perturbed output. However, due to the change in the sensitivity
level, the privacy guarantees are slightly worse
than the standard case. For example, as explained in
\cite{Nissim2007}, drawing the noise according to any sub-gaussian
distribution can not yield pure differential privacy. If one insists
on obtaining pure differential privacy, more heavy-tailed
distributions such as Cauchy distribution should be used. We discuss
one non-pure (and less noisy) and one pure (and more noisy) possibilities.
\begin{theorem}  \textbf{(\cite{Nissim2007}[Lemmas 2.7 and 2.10])} \label{thm:nissim}\\
Let $\epspr, \deltapr \in (0,1)$ be the differential privacy
parameters. The following claims hold:
\begin{enumerate}
\item
\noindent \textbf{Gaussian noise: } Suppose that $U(S)$ is
a $\beta$-smooth upper bound on the local sensitivity, where $\beta =
\frac{\epspr}{4(d+\ln(2/\deltapr))}$. Define the noise matrix in
\algref{alg:dpPCA} by
\[
E_{\texttt{APPROX}} = \frac{5U(S) \cdot \sqrt{2 \ln(2/\delta)}}{\epspr} Z~,
\]
where $Z$ is a standard $d$-dimensional Gaussian random variable. Then
\algref{alg:dpPCA} is $(\epspr,\deltapr)$-differentially private.\\
\item
\noindent \textbf{Cauchy noise: } Suppose that $U(S)$ is
a $\beta$-smooth upper bound on the $\ell_1$-local sensitivity, where $\beta =
\frac{\epspr}{6d}$. Define the noise matrix in
\algref{alg:dpPCA} by

\[
E_{\texttt{PURE}} = \frac{6U(S)}{\epspr} Z~,
\]
where $Z_1,\ldots,Z_d$ are drawn i.i.d. according to the density
function $f(z) = \frac{1}{\pi(1+z^2)}$, is $\epspr$-differentially
private.
\end{enumerate}
\end{theorem}
\begin{remark}
Lemmas 2.7 and 2.10 in \cite{Nissim2007} refer to the noisy output
$\hU$ before the QR step. However, since differential privacy is immune to
post-processing (\cite{dwork2014algorithmic}), the claim holds for the
output $\tU$ as well.
\end{remark}

\subsection{Smooth Sensitivity of PCA}
In this part we bound the smooth sensitivity of PCA and establish
the privacy properties of \algref{alg:dpPCA}.
\begin{lemma} \label{lem:gapPreserve}
Let $S=(x_1,\ldots,x_n)$ be a sample of size $n$ and suppose that
$\GAP(S)>0$. For any sample $S' = (z_1,\ldots,z_m)$ with $d(S,S') \le k$, we have
that
\[
\max\{0, n \cdot \GAP (S) - k \}  \le m \cdot \GAP(S') \le  n \cdot \GAP (S) + k~.
\]
Furthermore, for each side of the above inequality, there exists a
sample $S'$ with $d(S,S') = k$ for which the inequality holds with
equality. 
\end{lemma}
\begin{proof}
Let $H=\sum_{i=1}^n x_i x_i^\top$, $M = \sum_{i=1}^m
z_i z_i^\top$ and denote by $P = H-M$. Using Weyl's
inequality (\cite{Bhatia1997}[Section 3.2]), we obtain that
\[
\lambda_1(M) \ge \lambda_1(H) + \lambda_d(P),~~~\lambda_2(M) \le \lambda_2(H) + \lambda_1(P)
\]
Since the $\ell_2$-norm of the
$x_i$'s (similarly, the $z_i$'s) is at most $1$ and $d(S,S') \le k$,  both the rank and the trace-norm
of $P$ are at most $k$.\footnote{The trace norm of $P$ is
$\sum_{i=1}^d |\lambda_i(P)|$. Since $P$ is the sum of $k$ rank-$1$
matrices of trace $1$, it follows using the triangle inequality that
the trace norm of $P$ is at most $k$.}Therefore,
\begin{align*}
\lambda_1(M) - \lambda_2(M) &\ge \lambda_1(H) - \lambda_2(H) +
\lambda_d(P) - \lambda_1(P) \\
&\ge \lambda_1(H) - \lambda_2(H) -
\sum_{i=1}^d |\lambda_i(P) | \\
&\ge  \lambda_1(H) - \lambda_2(H) - k~.
\end{align*}
This concludes the inequality. Letting $u_2$ be the second leading eigenvector of $H$, the right side of the inequality is attained by setting $S' = 
S + \sum_{i=1}^k u_2$. The left side is attained analogously. 
\end{proof}
Combining the last lemma with \thmref{thm:ell2SensitivityUpper}, \lemref{lem:nissim} and
\thmref{thm:nissim}, we conclude that \algref{alg:dpPCA} is
differentially private.
\begin{corollary} \textbf{(Approximate differential privacy)} \label{cor:approxDP}
Suppose that $\texttt{PURE} = \texttt{FALSE}$ and let 
\begin{equation} \label{eq:finalNoiseNonPure}
E_{\texttt{PURE}}=\frac{ 5 \max_{k \in [n]} \exp(-\beta k) A^{(k)}(S)
  \cdot \sqrt{2 \ln(2/\delta)}}{\epspr} Z~,
\end{equation}
where $Z$ is standard $d$-dimensional Gaussian random variable and
\[
A^{(k)}(S) = 
\begin{cases}   
\frac{C}{n \cdot \GAP(S) - k}  & n \cdot \GAP(S) - k > 0   \\
 \sqrt{2}  & \textrm{otherwise}
\end{cases}
\]
Then \algref{alg:dpPCA} is $(\epspr,\deltapr)$-differentially private.
\end{corollary}
\begin{corollary} \textbf{(Pure differential privacy)} \label{cor:pureDP}
Suppose that $\texttt{PURE} = \texttt{TRUE}$ and let
\begin{equation} \label{eq:finalNoisePure}
E_{\texttt{PURE}} = \frac{6 \max_{k \in [n]} \exp(-\beta
  k)A^{(k)}(S)}{\epspr} Z~,
\end{equation}
where $Z_1,\ldots,Z_d$ are i.i.d. Cauchy random variables and
\[
A^{(k)}(S) = \begin{cases}   
\frac{C\sqrt{d}}{n \cdot \GAP(S) - k}  & n \cdot \GAP(S) - k > 0   \\
2 & \textrm{otherwise}
\end{cases}
\]
Then \algref{alg:dpPCA} is $\epspr$-differentially private.
\end{corollary}

\subsection{Near-optimal accuracy under eigengap assumption}
In this part we complete the proof of our main result by bounding the
smooth sensitivity under the eigengap assumption. We start by
relating the distributional gap assumption to the empirical
eigengap. The following lemma follows from Matrix Bernstein
inequality (\cite{Tropp2015}). 
\begin{lemma} \label{lem:Bernstein}
Suppose that
$\GAP(\cD):=\lambda_1(\bE[xx^\top])-\lambda_2(\bE[xx^\top])>0$. If $n
= \Omega \left( \frac{\log(d)}{\GAP(\cD)^2} \right)$, then with
probability at least $1-\delta/2$ over the draw of a sample $S$
according to $\cD^n$, $\GAP(S) \ge \GAP(\cD)/2$.
\end{lemma}
The next two lemma refer to the $\ell_1$ and the $\ell_2$ cases, respectively.
\begin{lemma} \label{lem:usBoundApprox}
Let $\epspr, \deltapr, \epsgen \in (0,1)$, and let $S=(x_1,\ldots,x_n)
\in \supp(\cD^n)$ be a sample
with $\GAP(S)>0$. Define $A^{(k)}(S)$ as in \corref{cor:approxDP} and let $\beta = \frac{\epspr}{4(d+\ln(2/\deltapr))}$. Suppose that 
$$
n \ge \frac{2C \sqrt{d}}{\GAP(S) \epspr \epsgen} + \frac{8(d+\ln(2/\deltapr))\ln
  (\sqrt{2d}/(\epspr \epsgen))}{\GAP(S) \epspr}~.
$$
Then
\[
U(S) := \max_{k \in \bN} \exp(-\beta k) A^{(k)}(S) \le \epsgen \epspr /\sqrt{d}
\] 
\end{lemma}
\begin{proof}
Assume first that $k \le \frac{n \cdot \GAP(S)}{2}$. In particular,
this implies that $n \cdot \GAP(S) - k > 0$. Using that $n \ge
\frac{2C \sqrt{d}}{\GAP(S) \epspr \epsgen}$, it follows that
\begin{align*}
U(S)  &:= \exp(-\beta k) A^{(k)}(S) \le A^{(k)}(S) = \frac{C}{n  \cdot\GAP(S)-k} \\
& \le \frac{2C}{n  \cdot\GAP(S)}  \le \epsgen \epspr/\sqrt{d}~.
\end{align*}
Assume now that $k > \frac{n \cdot \GAP(S)}{2}$, so $\exp(-\beta k) \le \exp \left (-\frac{\beta n \cdot \GAP(S)}{2}
\right)$. Using that $n \ge  \frac{8(d+\ln(2/\deltapr))\ln
  (\sqrt{2d}/(\epspr \epsgen))}{\GAP(S) \epspr}$, we obtain that $\exp(-\beta k)  \le \epsgen\epspr/(\sqrt{2} d)$
Since $A^{(k)} \le \sqrt{2}$ for all $S$ and $k$,
\[
\exp(-\beta k) A^{(k)}(S) \le \epsgen \epspr/\sqrt{d}~.
\]
\end{proof}
We proceed to the $\ell_1$-case. 
\begin{lemma} \label{lem:usBoundPure}
Let $\epspr, \epsgen \in (0,1)$ and let $S=(x_1,\ldots,x_n)
\in \supp(\cD^n)$ be a sample
with $\GAP(S)>0$. Define $A^{(k)}(S)$ as in \corref{cor:pureDP} and let $\beta = \frac{\epspr}{6d}$. Suppose that 
$$
n \ge \frac{2Cd^{3/2}}{\GAP(S) \epspr \epsgen} + \frac{6d\, \ln
  (2d/(\epspr \epsgen))}{\GAP(S) \epspr}~.
$$
Then
\[
U(S) := \max_{k \in \bN} \exp(-\beta k) A^{(k)}(S) \le \epsgen \epspr /d~.
\] 
\end{lemma}
We finally conclude our main result.
\begin{proof} \textbf{(of \thmref{thm:mainApprox} and \thmref{thm:mainPure})}
The differential privacy of the algorithm was established in
\corref{cor:approxDP} and \corref{cor:pureDP}. We next prove the
bounds on the sample complexity. All the bounds given below hold with
constant probability. 

In view of \lemref{lem:Bernstein}, we may assume that $\GAP(S)$ is of
order $\GAP(\cD)$. Sample complexity bounds for PCA
(\cite{gonen2017fast, blanchard2007statistical}) show that for $n = \Omega
\left( \frac{1}{ \GAP(\cD) \epsgen} \right)$, the
true risk of any unit vector is $\epsgen/4$-close to its empirical
risk. Therefore, adopting the notation used in \algref{alg:dpPCA}, $F(\hu) \le \min F(u) + \epsgen/2$. It is left to
show that 
$$
\hF(\tu) - \hF(\hu) \le \epsgen/2~.
$$
For the case $k=1$, the QR decomposition step
amounts to normalizing the noisy vector $\overbar{u}$. Therefore, it
suffices to show bound the $\ell_2$ norm of the noise vector
$\overbar{u}-\hu$ by $\epsgen$. For approximate differential privacy, standard concentration
bounds give a bound of order $\sqrt{d}$ on the $\ell_2$ norm of a
standard $d$-dimensional Gaussian vector. Using
\lemref{lem:usBoundApprox}, we conclude the bound. For the pure
setting, it is known that the median of the absolute value of  a
Cauchy random variable is $1$. Since the Cauchy distribution is
$1$-stable, the sum of $d$ i.i.d. Cauchy random variables is also a
standard Cauchy random variable scaled by $d$. Consequently, the
$\ell_1$ and the $\ell_2$ norms of the corresponding vector can be
bounded by $d$ (with constant probability). The desired bound follows
from \lemref{lem:usBoundPure}.
\end{proof}

\section{Discussion}
In this work we studied the problem of adding privacy properties to the commonly used PCA algorithm. We showed that we can add privacy as a post processing step to any PCA solver while maintaining good accuracy. Moreover, the post processing step is efficient and preserves the utility of the PCA algorithm. This is a significant improvement over previous results that are either not computationally efficient or otherwise require changes to the implementations of PCA solvers.  

We believe that some of the techniques used in our paper may be
beneficial for other related problems. For example, our approach can be
applied to any strict saddle problem for which we are able to compute the
expression $A^{(k)}(S)$ which controls the smooth
sensitivity. Furthermore, our technique for overcoming symmetry
between equivalent solutions can be applied to most known strict
saddle problems such as low rank problems whose minima are unique up
to rotation (\cite{ge2017no}).

\section*{Acknowledgements} 
We would like to thank Kunal Talwar for sharing his ideas and
knowledge. 


\newpage
 \bibliography{bib}
\bibliographystyle{plain}



\end{document}